\documentclass[runningheads]{llncs}
\usepackage[misc]{ifsym}

\usepackage{hyperref}
\usepackage{url}
\usepackage{algorithm}
\usepackage{algorithmic}
\usepackage{comment}
\usepackage{color}
\usepackage{graphicx}
\usepackage{amsmath}
\usepackage{mathtools}
\usepackage{booktabs}
\usepackage{color}
\usepackage{bm}
\usepackage{algorithm}
\usepackage{algorithmic}

\DeclareMathOperator*{\htheta}{\hat{\theta}}
\usepackage{amssymb}\usepackage{multirow}
\usepackage{enumitem}
\usepackage{caption}
\usepackage{lipsum}
\usepackage[flushleft]{threeparttable}
\usepackage{mathtools} 
\usepackage{tikz}
\usepackage{graphicx}

\makeatother
\begin{document}
\title{Efficient Hyperdimensional Computing}
%
%

\author{Zhanglu Yan,  Shida Wang\textsuperscript{(\Letter)}, Kaiwen Tang, Weng-Fai Wong}
\tocauthor{Zhanglu Yan,  Shida Wang, Kaiwen Tang, Weng-Fai Wong}
\institute{National University of Singapore, Singapore \\ \texttt{ \{zhangluyan, tang\_kaiwen, wongwf\}@comp.nus.edu.sg,\{shida\_wang\}@u.nus.edu}}

\maketitle              

\begin{abstract}

Hyperdimensional computing (HDC) is a method to perform classification that uses binary vectors with high dimensions and the majority rule. This approach has the potential to be energy-efficient and hence deemed suitable for resource-limited platforms due to its simplicity and massive parallelism. However, in order to achieve high accuracy, HDC sometimes uses hypervectors with tens of thousands of dimensions. This potentially negates its efficiency advantage.
In this paper, we examine the necessity of such high dimensions and conduct a detailed theoretical analysis of the relationship between hypervector dimensions and accuracy. Our results demonstrate that as the dimension of the hypervectors increases, the worst-case/average-case HDC prediction accuracy with the majority rule decreases. Building on this insight, we develop HDC models that use binary hypervectors with dimensions orders of magnitude lower than those of state-of-the-art HDC models while maintaining equivalent or even improved accuracy and efficiency. For instance, on the MNIST dataset, we achieve 91.12\% HDC accuracy in image classification with a dimension of only 64. Our methods perform operations that are only 0.35\% of other HDC models with dimensions of 10,000. Furthermore, we evaluate our methods on ISOLET, UCI-HAR, and Fashion-MNIST datasets and investigate the limits of HDC computing~\footnote{https://github.com/zhangluyan9/EffHDC}. 

\keywords{Hyperdimension Computing \and Energy efficient computing.}
\end{abstract}
\section{Introduction}

{\em Hyperdimensional computing} (HDC) is a novel learning paradigm that takes inspiration from the abstract representation of neuron activity in the human brain. HDCs use high-dimensional binary vectors, and they offer several advantages over other well-known training methods like artificial neural networks (ANNs). One of the advantages of HDCs is their ability to achieve high parallelism and low energy consumption, which makes them an ideal choice for resource-constrained applications such as electroencephalogram detection, robotics, language recognition, and federated learning. Several studies have shown that HDCs are highly efficient in these applications ~\cite{hsieh2021fl,asgarinejad2020detection,neubert2019introduction,rahimi2016robust}. Moreover, HDCs are relatively easy to implement in hardware~\cite{schmuck2019hardware,salamat2019f5}, which adds to their appeal as a practical solution for real-world problems, especially in embedded devices.

Unfortunately, the practical deployment of HDC
suffers from low model accuracy and is always restricted to small and simple datasets. To solve the problem, one commonly used technique is increasing the hypervector dimension~\cite{neubert2019introduction,schlegel2022comparison,yu2022understanding}. For example, running on the MNIST dataset, hypervector dimensions of 10,000 are often used. \cite{duan2022lehdc} and \cite{yu2022understanding} achieved the state-of-the-art accuracies of 94.74\% and 95.4\%, respectively. 
In these and other state-of-the-art HDC works, hypervectors are randomly drawn from the hyperspace  $\{-1,+1\}^d$, where the dimension $d$ is very high. This ensures high orthogonality, making the hypervectors more independent and easier to distinguish from each other~\cite{thomas2020theoretical}. As a result, accuracy is improved and more complex application scenarios can be targeted. However, the price paid for the higher dimension is in higher energy consumption, possibly negating the advantage of HDC altogether~\cite{neubert2019introduction}. This paper addresses this trade-off and well as suggests a way to make use of it to improve HDC.


In this paper, we will analyze the relationship between hypervector dimensions and accuracy.
It is intuitively true that high dimensions will lead to higher orthogonality~\cite{thomas2020theoretical}. However, contrary to popular belief, we found that as the dimension of the hypervectors $d$ increases, the
upper bound for inference worst-case accuracy and average-case accuracy actually {\em decreases} (Theorem~\ref{prop1} and Theorem~\ref{prop2}). In particular, if the hypervector dimension $d$ is sufficient to represent a vector with $K$ classes (in particular, $d > \log_2 K$) then \textbf{the lower the dimension, the higher the accuracy.} 

Based on our analysis, we utilized the fully-connected network (FCN) with integer weight and binary activation as the encoder. Our research has shown that this encoder is equivalent to traditional HDC encoding methods, as demonstrated in Section ~\ref{sec:low-d-h-t}. Additionally, we will be learning the representation of each class through the majority rule. This will reduce the hypervector dimension while still maintaining the state-of-the-art accuracies.

When running on the MNIST dataset, we were able to achieve HDC accuracies of 91.12/91.96\% with hypervector dimensions of only 64/128. Also, the total number of calculation operations required by our method ($d=64$) was only 0.35\% of what was previously needed by related works that achieved the state-of-the-art performance. These prior methods relied on hypervector dimensions of 10,000 or more. Our analysis and experiments conclusively show that such high dimensions are not necessary.

The contribution of this paper is as follows:
\begin{itemize}
    \item We give a comprehensive analysis of the relationship between hypervector dimension and the accuracy of HDC. Both the worst-case and average-case accuracy are studied. Mathematically, we explain why relatively lower dimensions can yield higher model accuracies. 
    
    \item After conducting our analysis, we have found that our methods can achieve similar detection accuracies to the state-of-the-art, while using much smaller hypervector dimensions (latency). For instance, by utilizing a dimension of just 64 on the widely-used MNIST dataset, we were able to achieve an HDC accuracy of 91.12\%. 
    
    \item We have also confirmed the effectiveness of our approach on other datasets commonly used to evaluate HDC, including ISOLET, UCI-HAR, and Fashion-MNIST, achieving state-of-the-art accuracies even with quite low dimensions. Overall, our findings demonstrate the potential of our methods in reducing computational overhead while maintaining high detection accuracy.
    
\end{itemize}

\paragraph{Organisation}
This paper is organized as follows. For completeness, we first introduce the basic workflow and background of HDC. In Section~\ref{sec:Method}, we present our main dimension-accuracy analysis and two HDC retraining approaches. To evaluate the effectiveness of our proposed methods, we conduct experiments and compare our results with state-of-the-art HDC models in Section~\ref{sec:exp}. Finally, we discuss the implications of our findings and conclude the paper.


\section{Background}

Hyperdimensional computing (HDC) is a technique that represents data using binary hypervectors with dimensions typically ranging from 5,000 to 10,000. For example, when working with the MNIST dataset, each flattened image $x \in \mathbb{R}^{784}$ is encoded into a hypervector $r \in \mathbb{R}^d$ using a binding operation that combines value hypervectors $v$ with position vectors $p$ and takes their summation. 

Both these two hypervectors $\mathbf{v, p}$ are independently drawn from the hyperspace  $\{-1,+1\}^d$ randomly. Mathematically, we can construct representation $r$ for each image as followed:
\begin{equation}
\label{eq:hdc_encoding}
    r = \textrm{sgn} \left ((v_{x_0}\bigotimes p_{x_0} + v_{x_1}\bigotimes p_{x_1}+ \cdots + v_{x_{783}}\bigotimes p_{x_{783}}) \right ),
\end{equation}

\noindent
where the sign function `$\textrm{sgn}(\cdot)$' is used to binarize the sum of the hypervectors, returning either -1 or 1. When the sum equals to zero, $\textrm{sgn}(0)$ is randomly assigned either 1 or -1. In addition, the binding operation $\bigotimes$ performs element-wise multiplication between hypervectors. For instance, $[-1,1,1,-1]\bigotimes[1,1,1,-1] = [-1,1,1,1]$.


During training, hypervectors $r_1, r_2, ..., r_{60,000}$ belonging to the same class are added together. The resulting sum is then used to generate a representation $R_i$ for class $i$, using the "majority rule" approach. The data belonging to class $i$ is denoted by $C_i$.

\begin{equation}
\label{r_c}
R_i = \textrm{sgn} \left (\sum_{x \in C_i} r_i \right ).
\end{equation}

During inference, the encoded test image is compared to each class representation $R_c$, and the most similar one is selected as the predicted class. Various similarity measures such as cosine similarity, L2 distance, and Hamming distance have been used in previous works. In this work, we use the inner product as the similarity measure for binary hypervectors with values of -1 and 1, as it is equivalent to the Hamming distance, as noted in~\cite{frady2021computing}. 

\section{High dimensions are not necessary} 
\label{sec:Method}

Contrary to traditional results that suggest higher-dimensional models have lower error rates, the majority rule's higher representation dimension in HDC domain does not always lead to better results. Our research demonstrates that if a dataset can be linearly separated and embedded into a $d$-dimensional vector space, higher-dimensional representation may actually reduce classification accuracy. This discovery prompted us to explore the possibility of discovering a low-dimensional representation of the dataset. Our numerical experiments support our theoretical discovery, as the accuracy curve aligns with our findings.






\subsection{Dimension-accuracy analysis}

Based on the assumption that the dataset can be linearly-separably embedded into a $d$-dimensional space, we here investigate the dimension-accuracy relationship for the majority rule. 

We further assume the encoded hypervectors are uniformly distributed over a $d$-dimensional unit ball:
\begin{equation*}
    B^d = \{r \in \mathbb{R}^d \big | \|r\|_2 \leq 1\}.
\end{equation*}
Moreover, we assume that hypervectors $x$ are {\em linearly separable} and each class with label $i$ can be represented by $C_i$:
\begin{equation*}
    C_i = \{r \in \mathcal{X} | R_i \cdot r > R_j \cdot r, j \neq i\}, \quad 1 \leq i \leq K
\end{equation*}
where $R_i \in [0, 1]^d $ are support hypervectors that are used to distinguish classes $i$ from other classes. 

Selecting a sufficiently large $d$ to embed the raw data into a $d$-dimensional unit ball is crucial for this approach to work effectively. This assumption is reasonable because with a large enough $d$, we can ensure that the raw data can be accurately mapped into a high-dimensional space where the support hypervectors can distinguish between different classes. 

 Similarly, we define the prediction class  $\hat{C}_i$ by $\hat{R}_i$ as followed:
\begin{equation*}
    \hat{C}_i = \{r \in \mathcal{X} | \hat{R}_i \cdot r > \hat{R}_j \cdot r, j \neq i\}, \quad 1 \leq i \leq K.
\end{equation*}
When we apply the majority rule to separate the above hypervectors $x$, we are approximating $R_i$ with $\hat{R}_i$ in the sense of maximizing the prediction accuracy.
Here each $\hat{R}_i \in \{0, 1\}^d$ is a binary vector.

Therefore we define the worst-case $K$-classes prediction accuracy over hypervectors distribution $\mathcal{X}$ in the following expression:
\begin{equation*}
    Acc^w_{K, d} := \inf_{R_1, R_2, \dots, R_K} \sup_{\hat{R}_1, \hat{R}_2, \dots, \hat{R}_K} \mathbb{E}_r \bigg [ \sum_{i=1}^K \prod_{j \neq i} \mathbf{1}_{\{R_i \cdot r > R_j \cdot r\}} \mathbf{1}_{\{\hat{R}_i \cdot r > \hat{R}_j \cdot r\}} \bigg ].
\end{equation*}

\begin{theorem}
\label{prop1}
    Assume $K = 2$, as the dimension of the hypervectors $d$ increases, the worst-case prediction accuracy decreases with the following rate:
    \begin{align*}
        Acc^w_{2, d} & = 2 \inf_{R_1, R_2} \sup_{\hat{R}_1, \hat{R}_2} \mathbb{E}_r \bigg [ \mathbf{1}_{\{R_1 \cdot r > R_2 \cdot r\}} \mathbf{1}_{\{\hat{R}_1 \cdot r > \hat{R}_2 \cdot r\}} \bigg ] \\
        & = \inf_{R_1, R_2} \sup_{\hat{R}_1, \hat{R}_2} \bigg [ 1 - \frac{\arccos (\frac{(R_1 - R_2) \cdot (\hat{R}_1 - \hat{R}_2)}{ \|R_1-R_2\|_2 \|\hat{R}_1-\hat{R}_2\|_2}) }{ \pi } \bigg ] \\
        & = 1 - \frac{\arccos (\frac{1}{\sqrt{\sum_{j=1}^d (\sqrt{j} - \sqrt{j-1})^2}}) }{\pi} \to \frac{1}{2}, \qquad d \to \infty
    \end{align*}
    \item {The first equality is by the symmetry of distribution $\mathcal{X}$. The second equality is the evaluation of expectation over $\mathcal{X}$ and the detail is given in Lemma 1. For the third equality, the proof is given in  Lemma 3 and Lemma~\ref{lemma:inequality_2_for_mr} 4.}

\end{theorem}

In the next theorem, we further consider the average-case. Assume the prior distribution $\mathcal{P}$ for optimal representation is uniformly distributed: $R_1 ,... R_K \sim \mathcal{U}[0, 1]^d$. 
We can define the average accuracy with the following expression:
\begin{equation*}
    \overline{Acc}_{K, d} := \mathbb{E}_{R_1, R_2, \dots, R_K \sim \mathcal{P}} \sup_{\hat{R}_1, \hat{R}_2, \dots, \hat{R}_K} \mathbb{E}_r \bigg [ \sum_{i=1}^K \prod_{j \neq i} \mathbf{1}_{\{R_i \cdot r > R_j \cdot r\}} \mathbf{1}_{\{\hat{R}_i \cdot r > \hat{R}_j \cdot r\}} \bigg ].
\end{equation*}
\begin{theorem}
\label{prop2}
    Assume $K = 2$, as the dimension of the hypervectors $d$ increases, the average case prediction accuracy decreases:
    \begin{align*}
        \overline{Acc}_{K, d} & = 2\mathbb{E}_{R_1, R_2 \sim U[0, 1]^d} \sup_{\hat{R}_1, \hat{R}_2} \mathbb{E}_r \bigg [ \mathbf{1}_{\{R_1 \cdot r > R_2 \cdot r\}} \mathbf{1}_{\{\hat{R}_1 \cdot r > \hat{R}_2 \cdot r\}} \bigg ] \\
        & = \mathbb{E}_{R_1, R_2 \sim U[0, 1]^d} \sup_{\hat{R}_1, \hat{R}_2} \bigg [ 1 - \frac{\arccos (\frac{(R_1 - R_2) \cdot (\hat{R}_1 - \hat{R}_2)}{ \|R_1-R_2\|_2 \|\hat{R}_1-\hat{R}_2\|_2}) }{ \pi } \bigg ] \\
        & = \mathbb{E}_{R_1, R_2 \sim U[0, 1]^d} \bigg [ 1 - \frac{\arccos \big( \sup_{j=1}^d \frac{\sum_{i=1}^j |R_1 - R_2|_{(i)}}{\sqrt{j} \|R_1 - R_2\|} \big )}{\pi} \bigg ].
    \end{align*}
Here $|R_1 - R_2|_{(i)}$ denotes the $i$-th maximum coordinate for vector $|R_1 - R_2|$.
\end{theorem}


Since the exact expression for average-case accuracy is challenging to evaluate, we rely on Monte Carlo simulations. In particular, we sample $R_1$ and $R_2$ 1000 times to estimate the expected accuracy. 

We then present the curves of $Acc^w_{K, d}$ and $\overline{Acc}_{K, d}$ over a range of dimensions from 1 to 1000 in Figures ~\ref{worst} and ~\ref{average}, respectively. It is evident from these figures that the upper bound of classification accuracy decreases as the dimension of the representation exceeds the necessary dimension. This observation implies that a higher representation dimension is not necessarily beneficial and could even lead to a decrease in accuracy. It is easy to find that a high dimension for HDCs is not necessary for both the worst-case and average-case, the upper bound of accuracy will drop slowly when the dimension increases. 



\begin{figure*}[t] 
\begin{minipage}[t]{0.49\linewidth} 
\centering
\includegraphics[width = 0.99\linewidth]{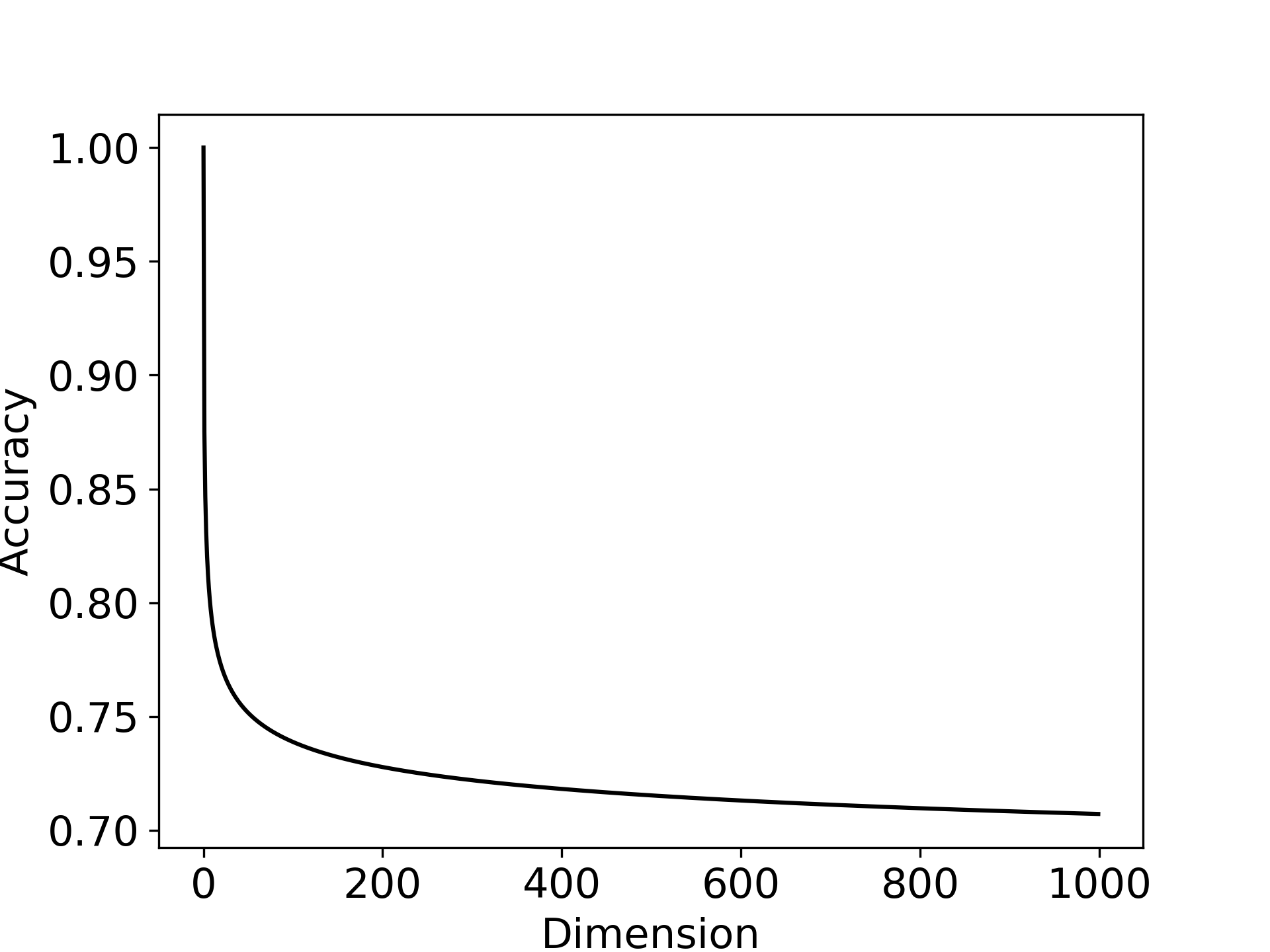}
\caption{ Worst-case Accuracy $Acc_{2, d}^w$} 
\label{worst} 
\end{minipage}
\begin{minipage}[t]{0.49\linewidth}
\centering
\includegraphics[width = 0.99\linewidth]{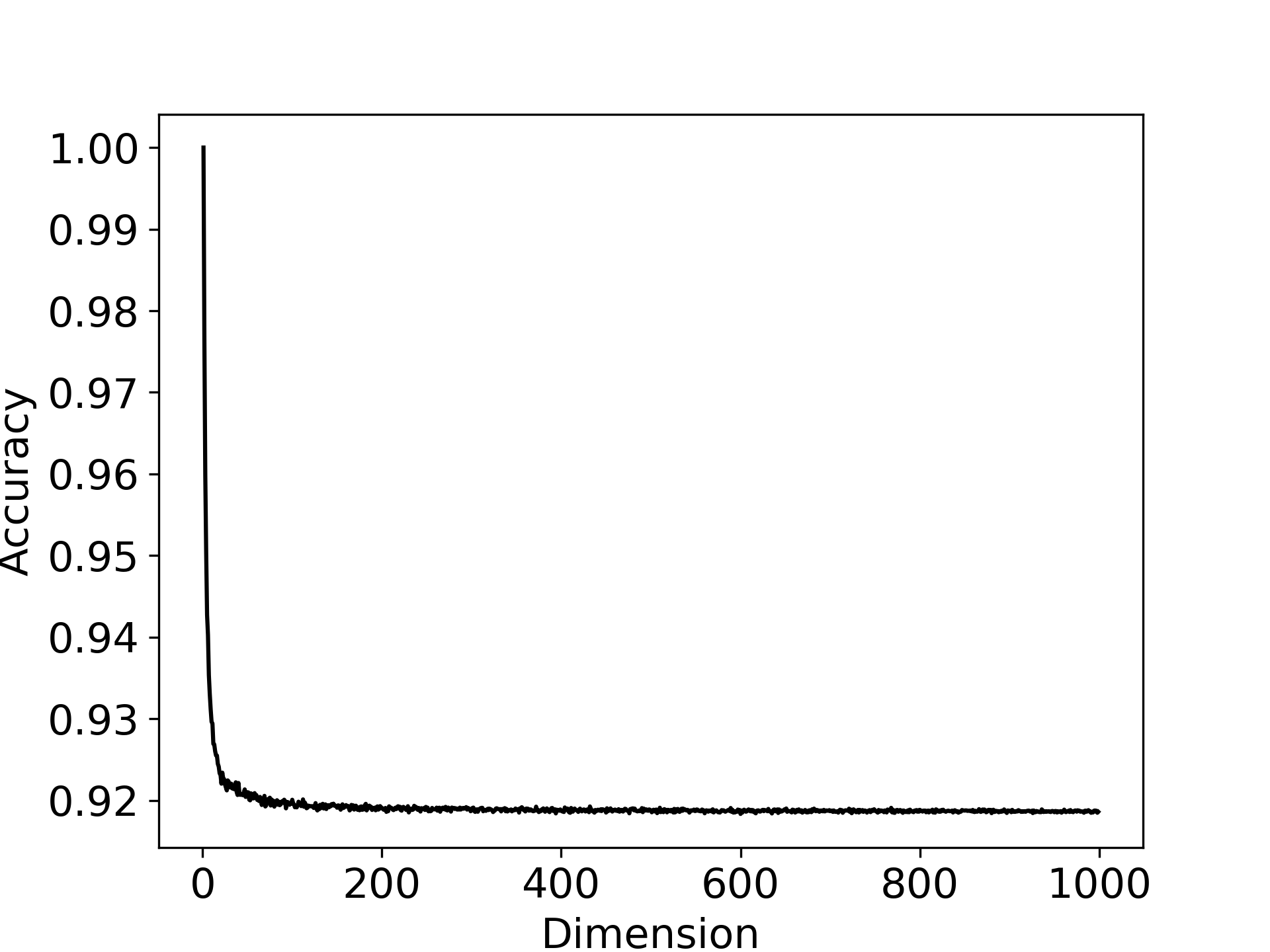}
\caption{Average-case Accuracy $\overline{Acc}_{2, d}$}
\label{average}
\end{minipage}
\end{figure*}

According to ~\cite{tax2002using}, we can approximate multi-class case where $K \geq 3$ by one-against-one binary classification. 
Therefore, we define the quasi-accuracy of $K$-class classification as follows:
\begin{equation*}
    Quasi\textrm{-}Acc_{K, d} = \frac{\sum_{i \neq j} Acc^{ij}_{2, d}}{K(K-1)},
\end{equation*}
where $Acc^{ij}_{2, d}$ can be either the average-case or worst-case accuracy that distinguishes class $i$ and $j$. Since the accuracy $Acc^{ij}_{2, d}$ for binary classification decreases as the dimension increase, the quasi-accuracy follows the same trend.

\subsection{Low-dimension Hypervector Training}
\label{sec:low-d-h-t}
\begin{figure}[htb]
    \centering
    \includegraphics[width = 1.0\linewidth]{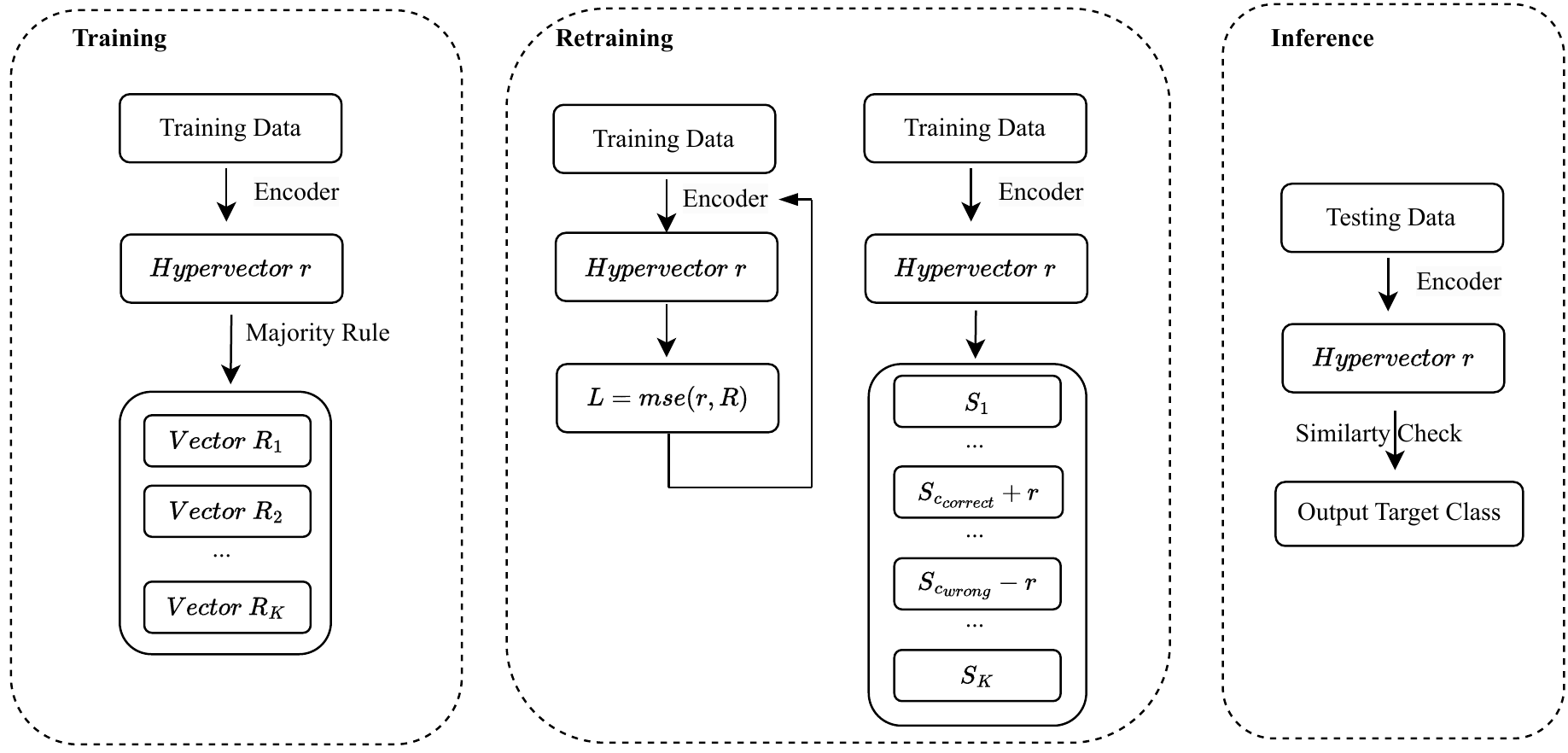}
    \caption{Workflow of Our HDC.}
    \label{fig:workflow}
\end{figure}

To confirm the theoretical findings mentioned above, we propose a HDC design that is shown in Figure~\ref{fig:workflow}. For data encoding, the traditional hyperdimensional computing technique utilizes binding and bundling operations to encode data samples using Equation~\ref{eq:hdc_encoding}. However, in this study, we use a simple binary fully-connected network with integer weights and binary activations as the encoder. Taking the MNIST dataset as an example, we demonstrate the equivalence of these two methods as follows:



\begin{align}
\label{eq:fcn-equivalent-to-hdc-encoder}
    r= \textrm{sgn}(Wx)= \textrm{sgn} \left ( \sum_{0 \leq i \leq 783} W_{i,x_i=1}\right ) 
\end{align}

\noindent
where $W_{i,x_i=1}$ indicates the weights whose  corresponding input $x_i$ =1.

The equation~\ref{eq:fcn-equivalent-to-hdc-encoder} shows that the sum of the weights $W_i$ corresponding to input $x_i=1$, while ignoring weights for $x_i=0$. The resulting sum of weight $\sum_{0 \leq i \leq 783} W_{i,x_i=1}$ in the linear transform corresponds to the sum of binding values of hypervectors $v$ and $p$ in Equation~\ref{r_c}. The integer-weights FCN with binary activation is a natural modification of the hyperdimensional computing encoders, using only integer additions, as in traditional HDC encoders.

Specifically, in our one-layer integer-weight fully-connected network, if we randomly initialize the weights with binary values, it becomes equivalent to the encoder of HDC.

 We used a {\em straight-through estimator} (STE) to learn the weights~\cite{bengio2013estimating} (details of STE are discussed in the Appendix~\footnote{https://github.com/zhangluyan9/EffHDC/blob/main/Appendix.pdf}). The binary representation $R_c$ of each class is generated using the majority rule (Algorithm~\ref{alg_mr}). To achieve this, we first sum up the $N$ hypervectors $r$ belonging to class $c$ and obtain an integer-type representation $S_c$ for that class. Subsequently, we assign a value of 1 if the element in $S_c$ exceeds a predefined threshold $\theta$. Otherwise, we set it to 0. This generates a binary representation $R_c$.



We have also devised a two-step retraining method to refine the binary representation $R_c$ to improve the accuracy. Algorithm~\ref{alg_tt} outlines the procedure we follow. First, we feed the training data to the encoder in batches and employ the mean squared error as the loss function to update the weights in the encoder. Next, we freeze the encoder and update the representation of each class. If the output $r$ is misclassified as class $c_{wrong}$ instead of the correct class $c_{correct}$, we reduce the sum of representation of the wrong class $S_{c_{wrong}}$ by multiplying $r$ with the learning rate. Simultaneously, we increase the sum of representation of the correct class $S_{c_{right}}$ by multiplying $r$ with the learning rate. We then use the modified $S_c$ in Algorithm~\ref{alg_mr} to generate the binary representation $R_c$.

\begin{minipage}{0.46\textwidth}
\begin{algorithm}[H]
    \centering
    \caption{Representation Generation:}
    \label{alg_mr}
    \begin{algorithmic}[1]
\REQUIRE $N$ number of training data $\bm{x}$;
\ENSURE Trained binary encoder $E$; Sum of representation $S$; Binary Representation $R_c$; Outputs of encoder $\bm{y}$; Pre-defined Threshold $\theta$; 

\STATE  $\bm{r} = E(\bm{x})$; $S_c = 0$
\FOR{$i=1$ to $N$}
\STATE $S_c+=r$
\ENDFOR
\FOR{$i=1$ to $d$}
\IF {$S_c[i]>\theta$}
\STATE $R_c[i]=1$
\ELSE
\STATE $R_c[i]=0$
\ENDIF
\ENDFOR

\end{algorithmic}
\end{algorithm}
\end{minipage}
\hfill
\begin{minipage}{0.46\textwidth}
\begin{algorithm}[H]

    \centering
    \caption{HDC Retraining:}
    \label{alg_tt}
\begin{algorithmic}[1]
\REQUIRE Training data $\bm{x}$ with label $\bm{R_c}$; Trained Encoder $E$; $N$ training epochs.

~

\STATE \textbf{Step1:}
\FOR{epoch$ =1$ to $N$}
\STATE $r$ = E($\bm{x}$)
\STATE $L$ = mse($r$, $R_c$) //Bp: STE
\ENDFOR

~

\STATE \textbf{Step2:}
\STATE $r$ = E($\bm{x}$)
\IF {Misclassified} 
\STATE $S_{c_{correct}}+=lr*r$
\STATE $S_{c_{wrong}}-=lr*r$
\ENDIF
\STATE Generate $R_c$ (Algorithm 1, line 5-9)
\end{algorithmic}
\end{algorithm}
\end{minipage}


After computing the representation of each class, we can compare the similarity between the resulting hypervector and the representation of all classes. To do this, we send the test data to the same encoder and obtain its hypervector representation. Next, we convert the value of 0 in $R_c$ to -1 and perform an inner product to check for similarity. The class with the highest similarity is reported as the result.

\section{Results}
\label{sec:exp}

We have implemented our schemes in CUDA-accelerated (CUDA 11.7) PyTorch version 1.13.0. The experiments were performed on an Intel Xeon E5-2680 server with two NVIDIA A100 Tensor Core GPUs and one GeForce RT 3090 GPU, running 64-bit Linux 5.15.  MNIST dataset\footnote{http://yann.lecun.com/exdb/mnist/}, Fashion-MNIST\footnote{https://github.com/zalandoresearch/fashion-mnist}, ISOLET\footnote{https://archive.ics.uci.edu/ml/datasets/isolet} and UCI-HAR\footnote{https://archive.ics.uci.edu/ml/datasets/human+activity+recognition+using+smartphones} are used in our experiments.

\subsection{A case study of our technologies}
Here, we will describe how our approaches improve the MNIST digit recognition task step by step. 

\begin{figure*}[ht] 
\caption{Threshold Study  } 
\label{d_thre}
\begin{minipage}[t]{0.32\linewidth} 
\centering
\includegraphics[width = 1\linewidth]{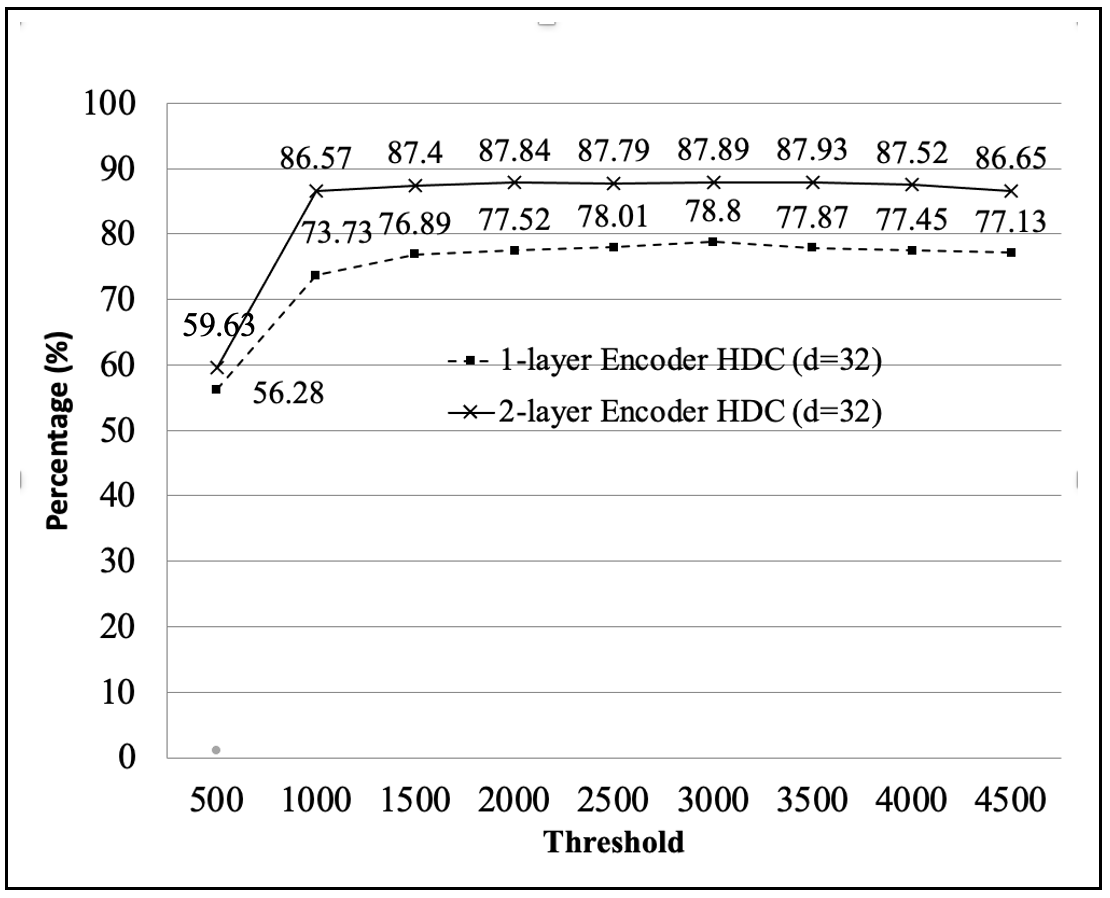}
\end{minipage}
\begin{minipage}[t]{0.32\linewidth}
\centering
\includegraphics[width = 1\linewidth]{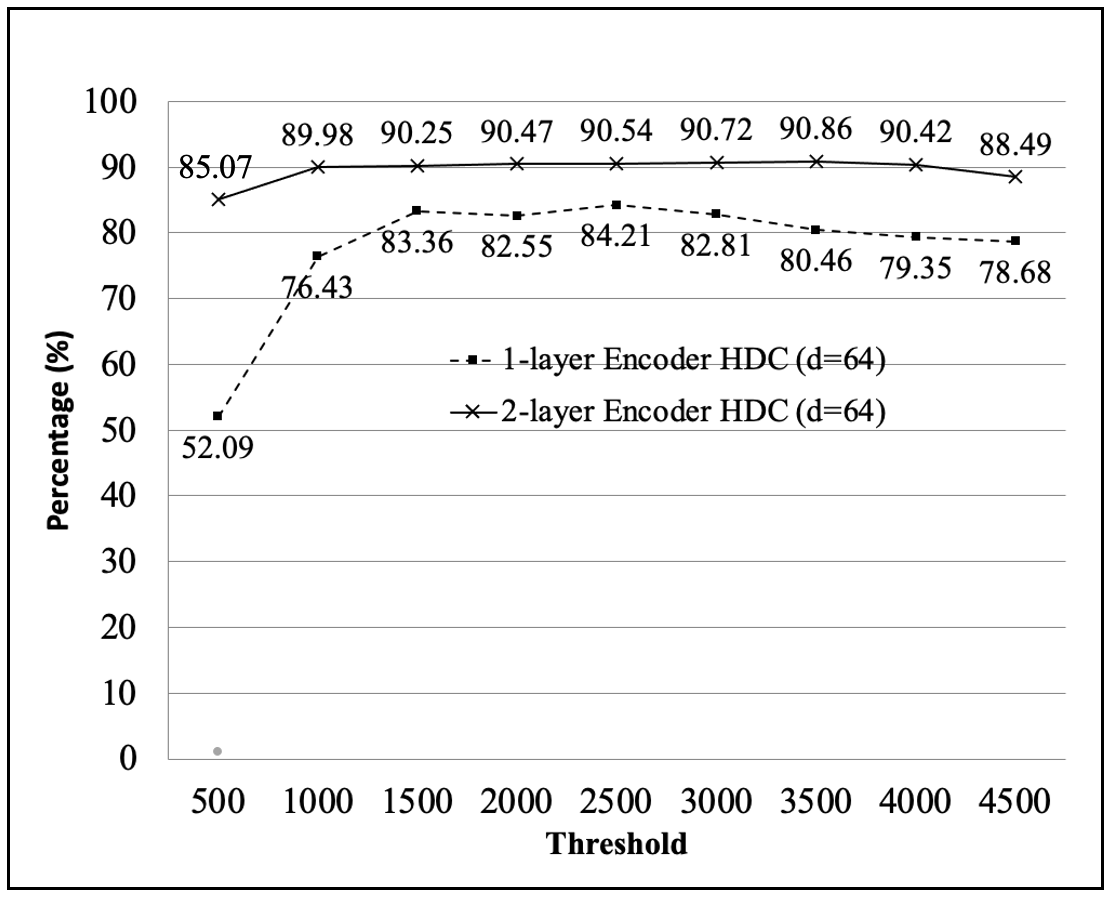}
\end{minipage}
\begin{minipage}[t]{0.32\linewidth}
\centering
\includegraphics[width = 1\linewidth]{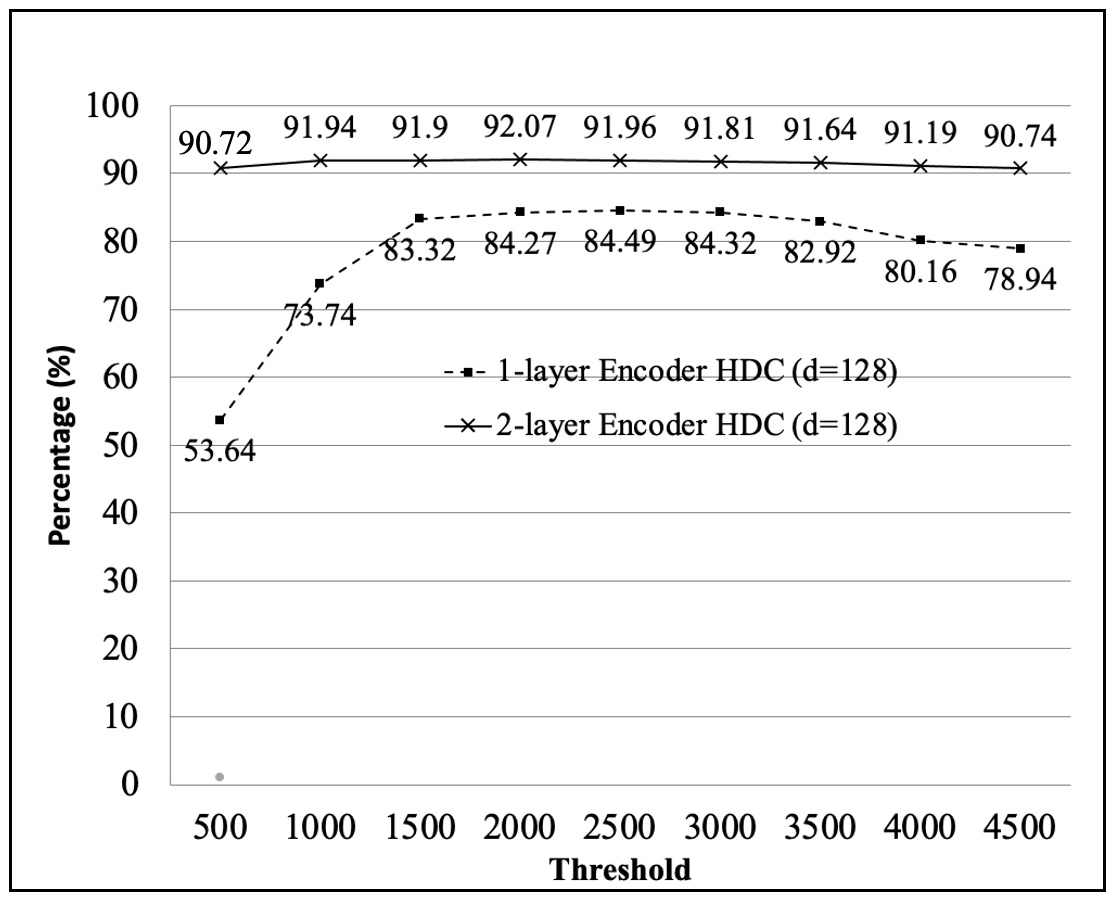}
\end{minipage}
\end{figure*}


\begin{figure*}[ht] 
\begin{minipage}[t]{0.32\linewidth} 
\centering
\includegraphics[width = 1\linewidth]{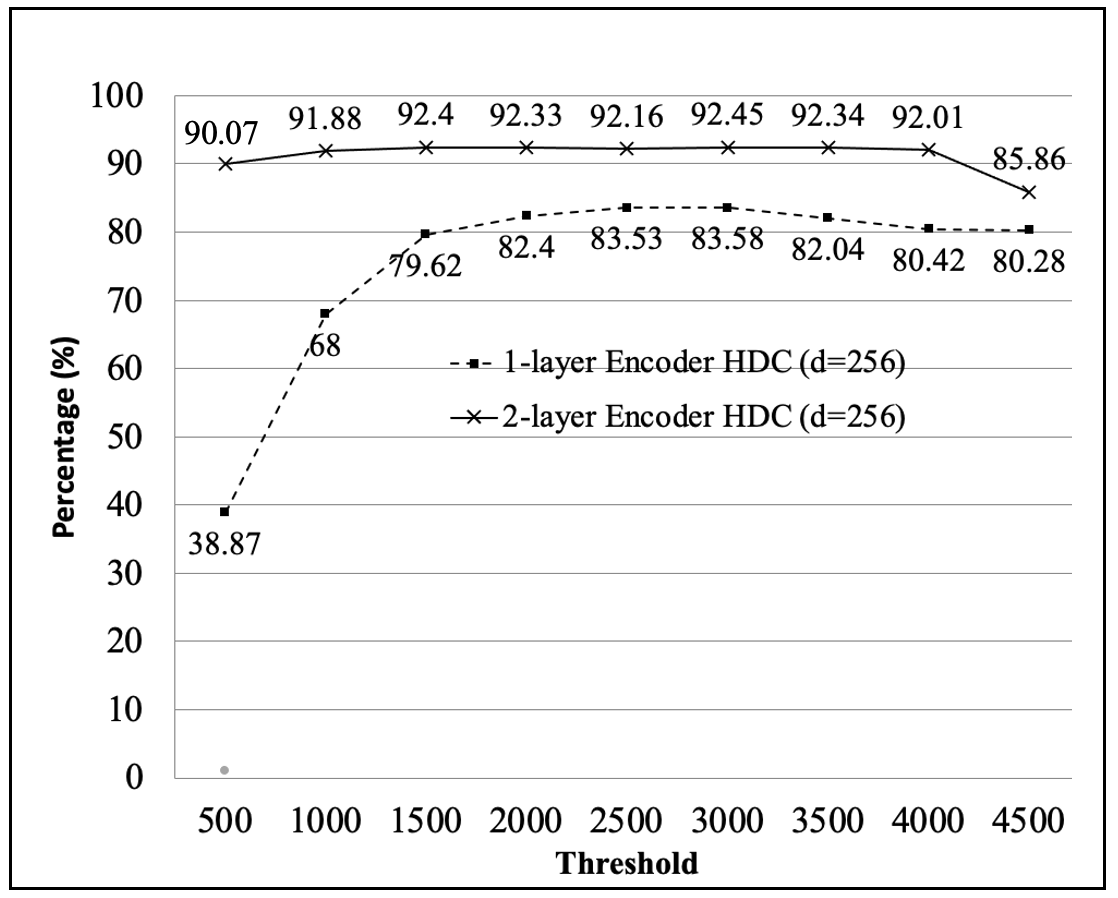}
\end{minipage}
\begin{minipage}[t]{0.32\linewidth}
\centering
\includegraphics[width = 1\linewidth]{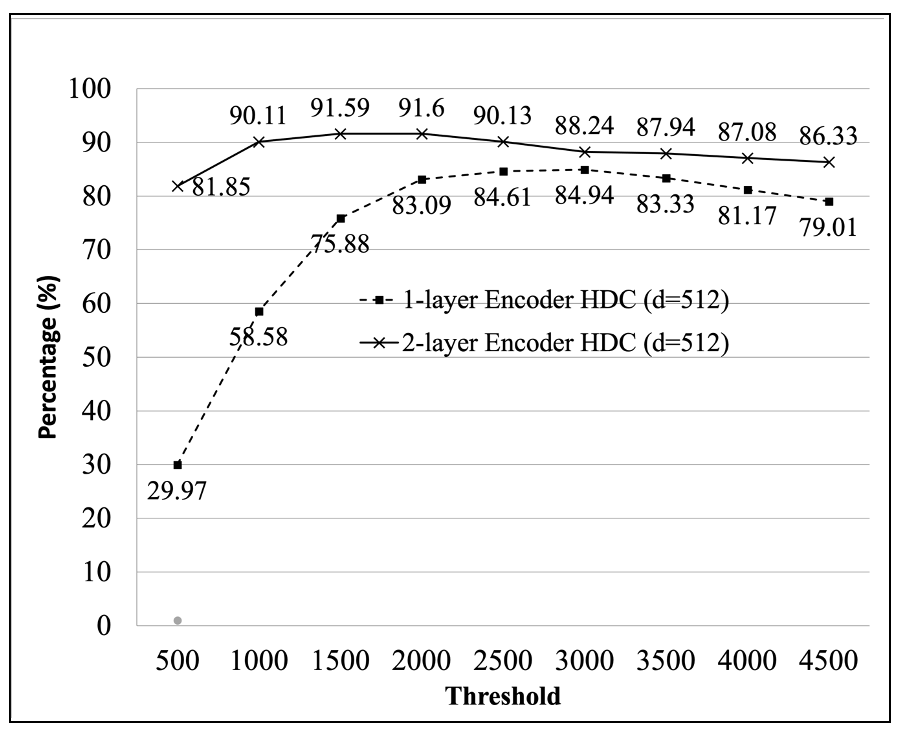}
\end{minipage}
\begin{minipage}[t]{0.32\linewidth}
\centering
\includegraphics[width = 1\linewidth]{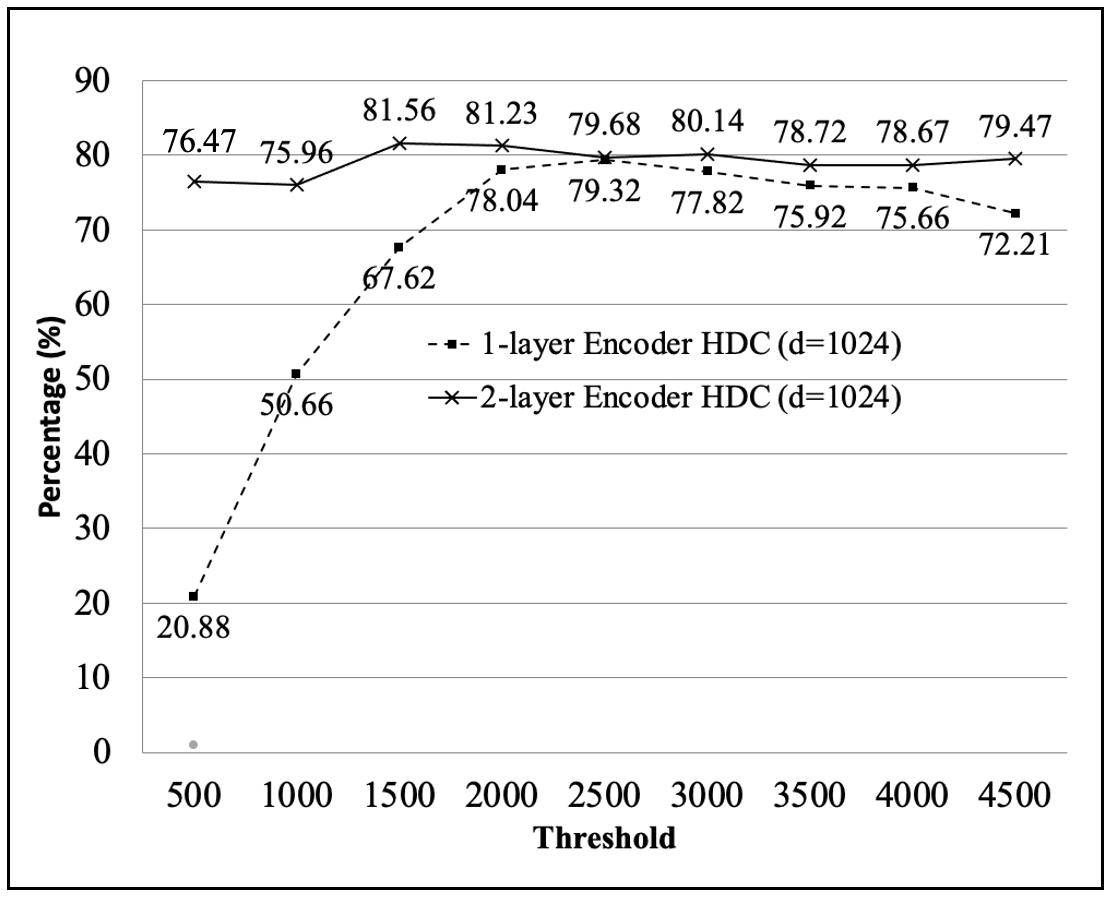}
\end{minipage}

\end{figure*}

\subsubsection{Baseline Accuracy}
Equation~\ref{eq:fcn-equivalent-to-hdc-encoder} shows that a single integer-weights FCN with binary activation can be transformed into traditional HDC encoding. To evaluate the performance of this transformation, we constructed two models consisting of one and two FCN layers (stack of one-layer FCN encoder) respectively and used them to encode an image. 
We then compared the results of both models. 

Using a dimension of 64 as an example, we investigate the correlation between the pre-defined threshold described in Algorithm~\ref{alg_mr} and accuracy. Our experiments on the MNIST dataset (shown in Figure~\ref{d_thre}) reveal that the threshold exhibits high robustness against noise even when the dimension is low. In fact, we observe that the detection accuracy remains virtually unchanged when the threshold was varied from 1000 to 4500. The maximum number in $S_c$ after the encoder is approximately 6500.

We further examined the connection between dimension and inference accuracy using the optimal threshold. According to Figures ~\ref{d_acc} and ~\ref{d_acc_}, we can attain HDC accuracies of 78.8\% and 84.21\%, as well as 87.93\% and 90.86\%, for 1-layer and 2-layer Encoder HDCs, respectively, with dimensions of only 32 and 64. Moreover, as stated in Theorem\ref{prop1}, the accuracy declines beyond a dimension of 128/256. Therefore, the accuracy of HDC is affected by the dimension, and we observed a consistent pattern with our previous findings.

\begin{figure*}[ht] 
\begin{minipage}[ht]{0.49\linewidth} 
\centering
\includegraphics[width = 0.99\linewidth]{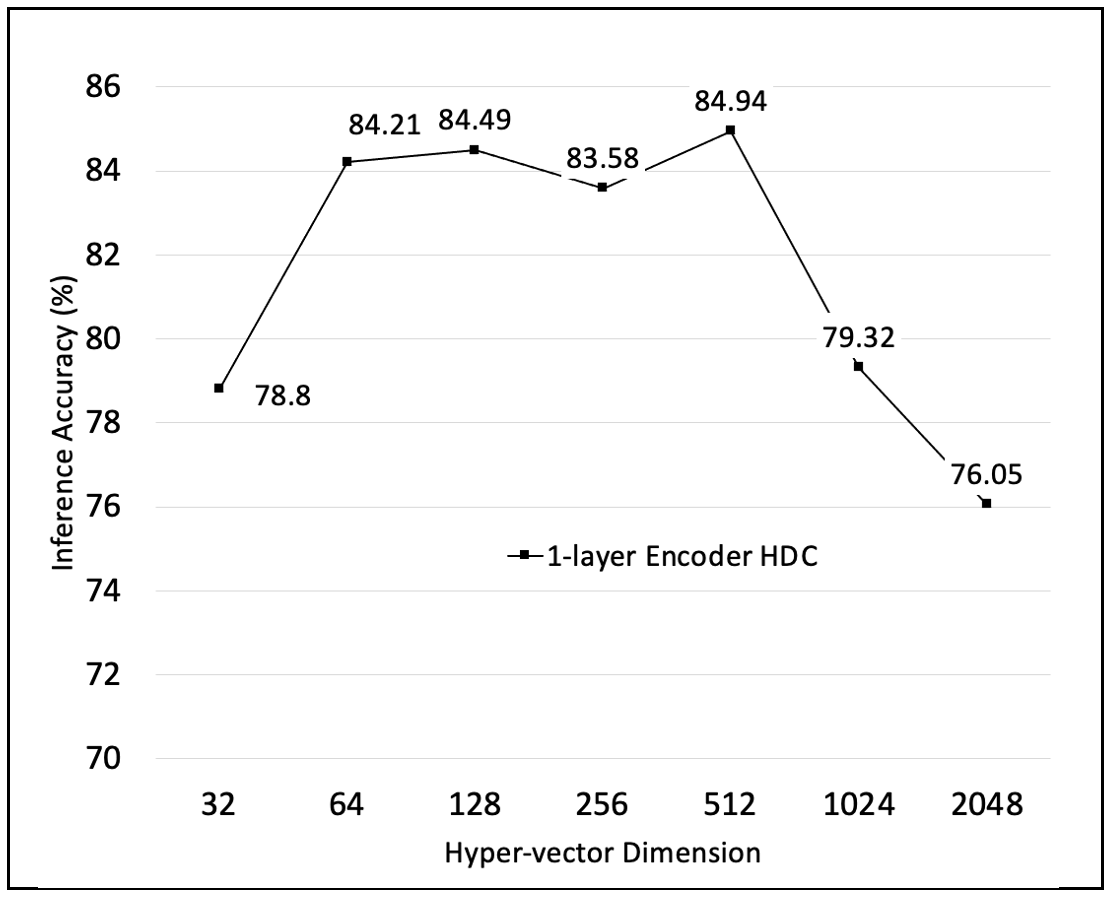}
\caption{1-layer Encoder HDC Accuracy } 
\label{d_acc} 
\end{minipage}
\begin{minipage}[ht]{0.49\linewidth}
\centering
\includegraphics[width = 1.0\linewidth]{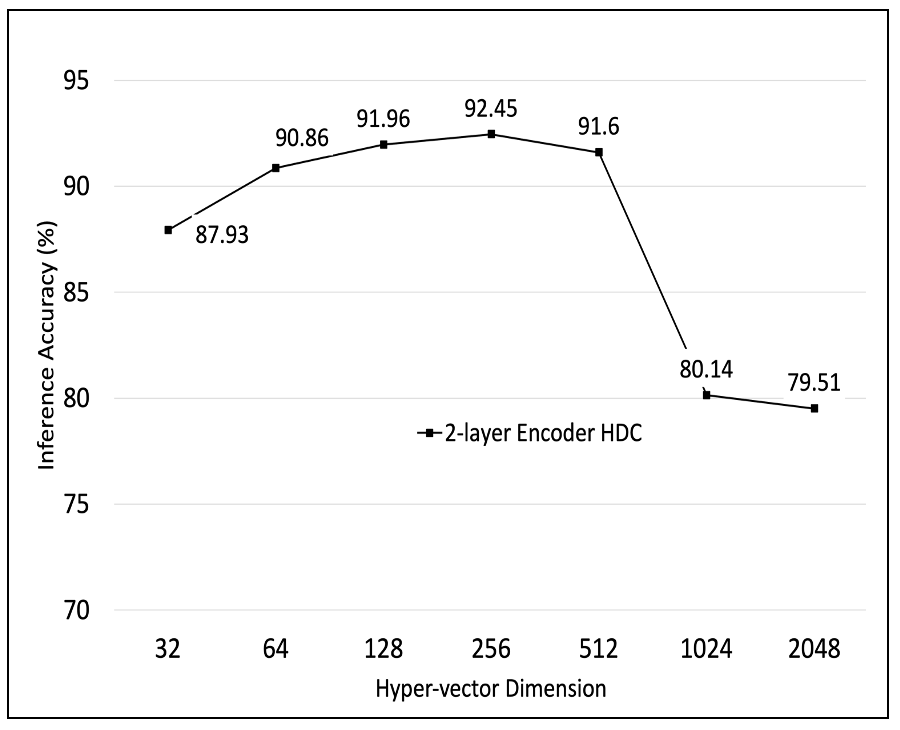}
\caption{2-layer Encoder HDC Accuracy}
\label{d_acc_}
\end{minipage}
\end{figure*}

\subsubsection{HDC Retraining}

Thus far, we have shown how we can achieve HDC accuracy of over 90\% with the smallest hypervector dimension. We can in fact improve the results using retraining techniques we will describe in this section. For example, with a dimension of 64, we can push the accuracy to 91.12\% with our two-step training (0.16\% and 0.1\% accuracy improvement with steps 1 and 2, respectively). The final accuracy can be increased in a matter of minutes.

\subsection{Experimental Results}


We present our final experimental findings with all previous technologies in Table~\ref{tab:comparison}, which showcases a comprehensive comparison between our HDC model and other state-of-the-art models in terms of accuracy, dimension, and number of operations. Our HDC model achieved accuracies of 84.21\% and 84.49\% for the MNIST dataset with encoder dimensions of only $d=64$ and $d=128$, respectively. We further improved the HDC accuracies to 91.12\% and 91.96\% by stacking an additional layer and batch normalization to the encoder architecture. These results demonstrate the effectiveness of our proposed HDC model in achieving competitive accuracies but much low dimension and computations in comparison to other state-of-the-art models.



\begin{table*}[t]
\centering
\caption{Comparison on MNIST dataset.}
\label{tab:comparison}
\begin{center}
\begin{tabular}{l|cccc}
\hline
 & \multirow{2}{*}{Accuracy}  &  \multirow{2}{*}{Dimension} & \multicolumn{2}{c}{Inference}\\ \cline{4-5}

           & & & Encoder additions / Boolean op  & Similarity \\
\hline\hline
\multicolumn{5}{c}{MNIST}                 \\ \hline

SearcHD & 84.43\%  & 10,000 & 7.84M/7.84M & Hamming   \\ \hline
FL-HDC & 88\%  & 10,000  & 7.84M/7.84M & Cosine   \\ \hline
TD-HDC & 88.92\%  & 5,000 & 3.92M/3.92M & Hamming  \\ \hline
QuantHD & 89.28\%  & 10,000 & 7.84M/7.84M &  Hamming  \\ \hline
LeHDC & 94.74\% & 10,000 & 7.84M/7.84M &  Hamming  \\ \hline

 \textbf{Ours*} & 84.21\%  & \textbf{64} & \textbf{0.05M/0} & Hamming   \\ \hline
 \textbf{Ours**} & \textbf{91.12\%}  & \textbf{64} & \textbf{0.054M/0} & Hamming   \\ \hline\hline

\end{tabular}
\end{center}
\begin{tablenotes}
\item {Ours* uses one-layer encoder while Ours** use two-layer encoder.}
\end{tablenotes}
\end{table*}

In order to compare the performance of our HDC model with the state-of-the-art, we selected several relevant works. One such work is TD-HDC, which was proposed by~\cite{chuang2020dynamic}. They developed a threshold-based framework to dynamically choose an execution path and improve the accuracy-energy efficiency trade-off. With their pure binary HD model, they achieved an HDC accuracy of 88.92\% on the MNIST dataset. In another case study,~\cite{hassan2021hyper} utilized a basic HDC model on the MNIST dataset. They encoded the pixels based on their black/white value and used majority sum operation in the training stage to combine similar samples. Their approach resulted in an HDC accuracy of 86\% on the MNIST dataset.


HDC has also been applied in the fields of federated learning and secure learning. For example, FL-HDC by~\cite{hsieh2021fl} focused on the combination of HDC and federated learning. They introduced the polarized model into the federated learning field to reduce communication costs and were able to control the accuracy drop by retraining. Their approach achieved an HDC accuracy of 88\% on the MNIST dataset. In another work, SecureHD~\cite{imani2019framework} adapted a novel encoding and decoding method based on HDC to perform secure learning tasks. 


Recent works in the field of HDC include LeHDC~\cite{duan2022lehdc}, which transferred the HDC classifier into a binary neural network, achieving accuracies of 94.74\% on the MNIST dataset. Additionally, QuantHD~\cite{imani2019quanthd} and SearcHD~\cite{imani2019searchd} are two methods that introduce multi-model and retraining techniques into the HDC field. These methods have shown promising results in improving the accuracy and performance of HDC models.

\begin{table*}[bt]
\centering
\caption{Comparison on other datasets.}
\label{tab:comparison1}
\begin{center}
\begin{tabular}{l|cccc}
\hline
 & \multirow{2}{*}{Accuracy}  &  \multirow{2}{*}{Dimension} & \multicolumn{2}{c}{Inference}\\ \cline{4-5}

           & & & Encoder additions/Boolean op & Similarity \\
\hline\hline
\multicolumn{5}{c}{Fashion-MNIST}                 \\ \hline

 BinHD & NA  & NA & NA & NA   \\ \hline
 \textbf{Ours} & \textbf{81.64\%}  & \textbf{64} & \textbf{0.059M/0} & Hamming   \\ \hline
 \textbf{Ours} & \textbf{81.58\%}  & \textbf{128} & \textbf{0.134M/0} & Hamming   \\ \hline\hline
 \multicolumn{5}{c}{ISOLET}                 \\ \hline

BinHD & 85.6\%  & 10,000 & 6.17M/6.17M & Hamming   \\ \hline
 \textbf{Ours} & \textbf{91.4\%}  & \textbf{64} & \textbf{0.249M/0} & Hamming   \\ \hline
 \textbf{Ours} & \textbf{93.2\%}  & \textbf{128} & \textbf{0.524M/0} & Hamming   \\ \hline\hline
 \multicolumn{5}{c}{UCI-HAR}                 \\ \hline

BinHD & 87.3\%  & 10,000 & 5.61M/5.61M & Hamming   \\ \hline
 \textbf{Ours} & \textbf{94.20\%}  & \textbf{64} & \textbf{0.336M/0} & Hamming   \\ \hline
 \textbf{Ours} & \textbf{94.81\%}  & \textbf{128} & \textbf{0.304M/0} & Hamming   \\ \hline\hline

\end{tabular}
\end{center}
\end{table*}

We also conducted tests on more datasets such as Fashion-MNIST, ISOLET, and UCI-HAR to assess the effectiveness of our methods. Specifically, we achieved an accuracy rate of 81.64\%, 91.4\%, and 94.20\% for Fashion-MNIST, ISOLET, and UCI-HAR, respectively, using only 64 dimensions for the hypervector. Our results are shown in Table~\ref{tab:comparison1}, with BinHD~\cite{imani2019binary} used as the baseline for comparison.

To minimize computational costs in our method, we decided to use Hamming distance for inference, as it is more efficient than cosine similarity, which involves additional multiplication and division operations. With Hamming distance, the number of operations is directly proportional to the dimension of the hypervectors. This means that our approach requires only 0.64\% of the operations needed by other HDC models with a dimension of 10,000, when using a dimension of $d$=64. This reduction in operations can speed up the inference process, making our approach more efficient for real-world applications.


\section{Discussion}
\subsection{Limitation of HDCs}

In this section, we aim to shed light on the relationship between the dimension of hypervectors and the number of classes, which has been largely overlooked in other HDC studies. To illustrate this point, we use the MNIST dataset as an example. State-of-the-art works typically employ hypervectors with dimensions ranging from 5,000 to 10,000 to differentiate pixel values that span from 0 to 255 and do a 10-classes classification. However, with an increase in the number of classes to 100 or 1,000, more information is required for precise classification. For input data, if we apply quantization and employ a suitable encoder to distill information from the original image, it is theoretically feasible to operate with a considerably smaller dimension. However, the quantity of classes remains invariant. This provides an explanation for the challenges encountered by our method, and other HDC techniques, in achieving high performance on more intricate datasets such as Cifar100 and ImageNet where the number of classes is significantly larger.

\subsection{Further discussion of the low accuracy when $d$ is low}
As can be seen from Figure \ref{worst}, \ref{average} and Figure \ref{d_acc}/\ref{d_acc_}, the current Theorem \ref{prop1} and \ref{prop2} do not predict the low accuracy for dimension $d \leq 128/256$. 

The issue can be attributed to the breakdown of the assumption that data can be embedded in a $d$-dimensional linearly separable unit ball. Consider a different setup in that the underlying dimension for data is fixed to be $m$. Each class is defined to be:
\begin{equation*}
    C_i = \{r \in \mathbb{B}^m | R_i \cdot r > R_j \cdot r, j \neq i\}, \quad 1 \leq i \leq K.
\end{equation*}

Assume that the linear projection of data from $m$-dimensional linearly separable unit ball to $d$-dimensional ($d < m$) space in a coordinate-wise approach. It is equivalent to optimizing over the following hypervector set
\begin{equation*}
    R_{co_1, \dots, co_d} = \{ R | R_{i} \in \{0, 1\}, i \in \{co_1, \dots, co_d\}; R_{i} = 0, i \not \in \{co_1, \dots, co_d\} \}, 
\end{equation*}
Here $co_1, \dots, co_d$ are the coordinates index of the projected space.

The worst-case $K$-classes prediction accuracy of the $m$-dimensional data projected onto a $d$-dimensional subspace is 
\begin{align*}
    Acc^w_{K, m, d} & := \inf_{R_1, R_2, \dots, R_K \in [0, 1]^m} \sup_{co_1, \dots, co_d} \sup_{\hat{R}_1, \hat{R}_2, \dots, \hat{R}_K \in R_{co_1, \dots, co_d}} \\ & \mathbb{E}_r \bigg [ \sum_{i=1}^K \prod_{j \neq i} \mathbf{1}_{\{R_i \cdot r > R_j \cdot r\}} \mathbf{1}_{\{\hat{R}_i \cdot r > \hat{R}_j \cdot r\}} \bigg ] \\
    & \leq Acc^w_{K, m, d+1}\\
    & \leq Acc^w_{K, m}.
\end{align*}
The monotonicity comes from the fact that the two supremums are taken over a monotonic hypervector set $R$ sequence.

The following theorem summarizes the above reasoning:
\begin{proposition}
\label{prop_last}
    Assume the representation dimension $d \leq m - 1$, the classification accuracy increases monotonically as $d$ increase:
    \begin{equation}
        Acc^w_{K, m, d} \leq Acc^w_{K, m, d+1}, \quad d \leq m-1.
    \end{equation}
\end{proposition}



Both Theorem~\ref{prop1} and \ref{prop2} characterize the accuracy when $d \geq m$. Proposition~\ref{prop_last} describes the dimension-accuracy relationship for $d \leq m$. The above reasoning has been confirmed by our numerical experiments.

\section{Conclusion}
In this paper, we considered the dimension of the hypervectors used in hyperdimensional computing. We presented a detailed analysis of the relationship between dimension and accuracy to demonstrate that it is not necessary to use high dimensions to get a good performance in HDC.  
Contrary to popular belief, we proved that as the dimension of the hypervectors $d$ increases, the
upper bound for inference worst-case accuracy and average-case accuracy decreases. As a result, we reduce the dimensions from the tens of thousands used by the state-of-the-art to merely tens, while achieving the same level of accuracy. Computing operations during inference have been reduced to a tenth of that in traditional HDCs. Running on the MNIST dataset, we achieved an HDC accuracy of $91.12\%$ using a dimension of only 64. All our results are reproducible using the code we have made public.

\clearpage

\section{Ethical Statements}

We hereby assure that the following requirements have been met in the manuscript:

\begin{itemize}
\item All of the datasets used in this paper are open-source and have been made publicly available for use. These datasets have been carefully vetted to ensure that they do not contain any personal or sensitive information that could compromise the privacy of individuals. Therefore, there are no concerns about violating personal privacy when using these datasets for research or analysis.

\item  The paper is free of any potential use of the work for policing or military purposes.

\item This manuscript presents the authors' original work that has not been previously published elsewhere.

\item The paper reflects the author's research and analysis accurately and completely.

\item Co-authors and co-researchers who made significant contributions to the work are duly acknowledged.

\item The results are appropriately contextualized within prior and existing research.

\item Proper citation is provided for all sources used.

\item All authors have actively participated in the research leading to this manuscript and take public responsibility for its content.
\end{itemize}

\bibliographystyle{splncs04}
\bibliography{ref}

\begin{thebibliography}{10}
\providecommand{\url}[1]{\texttt{#1}}
\providecommand{\urlprefix}{URL }
\providecommand{\doi}[1]{https://doi.org/#1}

\bibitem{asgarinejad2020detection}
Asgarinejad, F., Thomas, A., Rosing, T.: Detection of epileptic seizures from
  surface {EEG} using hyperdimensional computing. In: 2020 42nd Annual
  International Conference of the IEEE Engineering in Medicine \& Biology
  Society (EMBC). pp. 536--540. IEEE (2020)

\bibitem{bengio2013estimating}
Bengio, Y., Léonard, N., Courville, A.: Estimating or propagating gradients
  through stochastic neurons for conditional computation (2013)

\bibitem{chuang2020dynamic}
Chuang, Y.C., Chang, C.Y., Wu, A.Y.A.: Dynamic hyperdimensional computing for
  improving accuracy-energy efficiency trade-offs. In: 2020 IEEE Workshop on
  Signal Processing Systems (SiPS). pp.~1--5. IEEE (2020)

\bibitem{duan2022lehdc}
Duan, S., Liu, Y., Ren, S., Xu, X.: {LeHDC}: Learning-based hyperdimensional
  computing classifier. arXiv preprint arXiv:2203.09680  (2022)

\bibitem{frady2021computing}
Frady, E.P., Kleyko, D., Kymn, C.J., Olshausen, B.A., Sommer, F.T.: Computing
  on functions using randomized vector representations. arXiv preprint
  arXiv:2109.03429  (2021)

\bibitem{hassan2021hyper}
Hassan, E., Halawani, Y., Mohammad, B., Saleh, H.: Hyper-dimensional computing
  challenges and opportunities for {AI} applications. IEEE Access  (2021)

\bibitem{hsieh2021fl}
Hsieh, C.Y., Chuang, Y.C., Wu, A.Y.A.: {FL-HDC}: Hyperdimensional computing
  design for the application of federated learning. In: 2021 IEEE 3rd
  International Conference on Artificial Intelligence Circuits and Systems
  (AICAS). pp.~1--5. IEEE (2021)

\bibitem{imani2019quanthd}
Imani, M., Bosch, S., Datta, S., Ramakrishna, S., Salamat, S., Rabaey, J.M.,
  Rosing, T.: {QuantHD}: A quantization framework for hyperdimensional
  computing. IEEE Transactions on Computer-Aided Design of Integrated Circuits
  and Systems  \textbf{39}(10),  2268--2278 (2019)

\bibitem{imani2019framework}
Imani, M., Kim, Y., Riazi, S., Messerly, J., Liu, P., Koushanfar, F., Rosing,
  T.: A framework for collaborative learning in secure high-dimensional space.
  In: 2019 IEEE 12th International Conference on Cloud Computing (CLOUD). pp.
  435--446. IEEE (2019)

\bibitem{imani2019binary}
Imani, M., Messerly, J., Wu, F., Pi, W., Rosing, T.: A binary learning
  framework for hyperdimensional computing. In: 2019 Design, Automation \& Test
  in Europe Conference \& Exhibition (DATE). pp. 126--131. IEEE (2019)

\bibitem{imani2019searchd}
Imani, M., Yin, X., Messerly, J., Gupta, S., Niemier, M., Hu, X.S., Rosing, T.:
  {SearcHD}: A memory-centric hyperdimensional computing with stochastic
  training. IEEE Transactions on Computer-Aided Design of Integrated Circuits
  and Systems  \textbf{39}(10),  2422--2433 (2019)

\bibitem{neubert2019introduction}
Neubert, P., Schubert, S., Protzel, P.: An introduction to hyperdimensional
  computing for robotics. KI-K{\"u}nstliche Intelligenz  \textbf{33}(4),
  319--330 (2019)

\bibitem{rahimi2016robust}
Rahimi, A., Kanerva, P., Rabaey, J.M.: A robust and energy-efficient classifier
  using brain-inspired hyperdimensional computing. In: Proceedings of the 2016
  international symposium on low power electronics and design. pp. 64--69
  (2016)

\bibitem{salamat2019f5}
Salamat, S., Imani, M., Khaleghi, B., Rosing, T.: {F5-HD}: Fast flexible
  {FPGA}-based framework for refreshing hyperdimensional computing. In:
  Proceedings of the 2019 ACM/SIGDA International Symposium on
  Field-Programmable Gate Arrays. pp. 53--62 (2019)

\bibitem{schlegel2022comparison}
Schlegel, K., Neubert, P., Protzel, P.: A comparison of vector symbolic
  architectures. Artificial Intelligence Review  \textbf{55}(6),  4523--4555
  (2022)

\bibitem{schmuck2019hardware}
Schmuck, M., Benini, L., Rahimi, A.: Hardware optimizations of dense binary
  hyperdimensional computing: Rematerialization of hypervectors, binarized
  bundling, and combinational associative memory. ACM Journal on Emerging
  Technologies in Computing Systems (JETC)  \textbf{15}(4),  1--25 (2019)

\bibitem{tax2002using}
Tax, D.M., Duin, R.P.: Using two-class classifiers for multiclass
  classification. In: 2002 International Conference on Pattern Recognition.
  vol.~2, pp. 124--127. IEEE (2002)

\bibitem{thomas2020theoretical}
Thomas, A., Dasgupta, S., Rosing, T.: Theoretical foundations of
  hyperdimensional computing. arXiv preprint arXiv:2010.07426  (2020)

\bibitem{yu2022understanding}
Yu, T., Zhang, Y., Zhang, Z., De~Sa, C.: Understanding hyperdimensional
  computing for parallel single-pass learning. arXiv preprint arXiv:2202.04805
  (2022)

\end{thebibliography}

\clearpage

\appendix
\section{Appendix}
\subsection{Straight-through Estimator}
The Straight-Through Estimator (STE)~\cite{bengio2013estimating} is a commonly used technique in deep learning paradigms that incorporate discrete or non-differentiable functions, for example, binarization. This method facilitates the approximation of gradients through these operations, thereby enabling end-to-end model training involving non-differentiable components. Consequently, we utilize STE for backpropagation and approximate its gradient $G'(x)$ correspondingly.
\begin{equation}
\label{eq:Q_gradient}
G'(x) \approx 1.
\end{equation}

\subsection{Proofs for Lemmas}
\begin{lemma}\label{lemma:equality_1_for_mr}
\begin{equation*}
    \mathbb{E}_x \bigg [ \mathbf{1}_{\{\theta_1 \cdot x > \theta_2 \cdot x\}} \mathbf{1}_{\{\hat{\theta}_1 \cdot x > \hat{\theta}_2 \cdot x\}} \bigg ] = \frac{1}{2} (1 - \frac{\arccos (\frac{(\theta_1 - \theta_2) \cdot (\hat{\theta}_1 - \hat{\theta}_2)}{ \|\theta_1-\theta_2\|_2 \|\hat{\theta}_1-\hat{\theta}_2\|_2})}{\pi}).
\end{equation*}
\end{lemma}
\begin{proof}
    Consider the plane spanned by vector $\theta_1-\theta_2$ and $\hat{\theta}_1-\hat{\theta}_2$ and the projection of $x$ to this plane, the two indicator function requires the angle $<Px, \theta_1 - \theta_2>$ and angle $<Px, \hat{\theta}_1 - \hat{\theta}_2>$ to be smaller than $\frac{\pi}{2}.$ 
    Evaluating the expectation over $\mathcal{X}$ is equivalent to evaluating the intersection region of two semicircles. Therefore the result is $\frac{\pi - \arccos (\frac{(\theta_1 - \theta_2) \cdot (\hat{\theta}_1 - \hat{\theta}_2)}{ \|\theta_1-\theta_2\|_2 \|\hat{\theta}_1-\hat{\theta}_2\|_2})}{2 \pi}$. 
\end{proof}

\begin{lemma}\label{lemma:equality_2_for_mr}
Let $\Delta \theta = \theta_1 - \theta_2, \Delta \hat{\theta} = \hat{\theta}_1 - \hat{\theta}_2$.
When the coordinates of vector $\Delta \theta$ are ordered by absolute value: 
$1 \geq |\Delta \theta_1| \geq |\Delta \theta_2| \geq \dots \geq |\Delta \theta_d|.$ 
Then we have the following equality:
\begin{equation*}
    \sup_{\Delta \hat{\theta} \in \{-1, 0, 1\}^d} \frac{\Delta \theta \cdot \Delta \hat{\theta}}{\|\Delta \theta\|_2 \|\Delta \hat{\theta}\|_2} = \sup_{1 \leq j \leq d} \{ \frac{\sum_{i=1}^j |\Delta \theta_j|}{\sqrt{j} \|\Delta \theta\|_2}\}.
\end{equation*}
\end{lemma}

\begin{proof}
    By the definition of the supremum, iterate over the list $\Delta \hat{\theta} \in [\mathbf{e}_1, \mathbf{e}_1+\mathbf{e}_2, \dots, \mathbf{e}_1+\mathbf{e}_2+\cdots+\mathbf{e}_d]$, $\mathbf{e}_i$ is the unit vector with the same sign as $\Delta \theta_i$, we know 
    \begin{equation*}
        \sup_{\Delta \hat{\theta} \in \{-1, 0, 1\}^d} \frac{\Delta \theta \cdot \Delta \hat{\theta}}{\|\Delta \theta\|_2 \|\Delta \hat{\theta}\|_2} \geq \sup_{1 \leq j \leq d} \{ \frac{\sum_{i=1}^j |\Delta \theta_j|}{\sqrt{j} \|\Delta \theta\|_2}\}.
    \end{equation*}
    
    Now we show the $\leq$ part. We show that when the $\Delta \theta$'s coordinates are ordered, the optimal $\Delta \hat{\theta}$ is of the form
    \begin{equation*}
        (\textrm{sign}(\Delta \theta_1), \dots, \textrm{sign}(\Delta \theta_j), 0, \dots, 0).
    \end{equation*}
    
    For any $\Delta \hat{\theta}$ with norm $\sqrt{j}$, 
    \begin{equation*}
        \Delta \theta \cdot \Delta \hat{\theta} \leq \sum_{i=1}^j |\Delta \theta_j|.
    \end{equation*}
    Therefore,
    \begin{equation*}
        \sup_{\Delta \hat{\theta} \in \{-1, 0, 1\}^d} \frac{\Delta \theta \cdot \Delta \hat{\theta}}{\|\Delta \theta\|_2 \|\Delta \hat{\theta}\|_2} = \sup_{j} \sup_{|\Delta \hat{\theta}| = \sqrt{j}, } \frac{\Delta \theta \cdot \Delta \hat{\theta}}{\|\Delta \theta\|_2 \|\Delta \hat{\theta}\|_2}  \leq \sup_{1 \leq j \leq d} \{ \frac{\sum_{i=1}^j |\Delta \theta_j|}{\sqrt{j} \|\Delta \theta\|_2}\}.
    \end{equation*}
\end{proof}

\begin{lemma}\label{lemma:inequality_1_for_mr}
\begin{align*}
    \inf_{\theta_1, \theta_2} \sup_{\hat{\theta}_1, \hat{\theta}_2} \bigg [ 1 - \frac{\arccos (\frac{(\theta_1 - \theta_2) \cdot (\hat{\theta}_1 - \hat{\theta}_2)}{ \|\theta_1-\theta_2\|_2 \|\hat{\theta}_1-\hat{\theta}_2\|_2}) }{ \pi } \bigg ] \leq 1 - \frac{\arccos (\frac{1}{\sqrt{\sum_{j=1}^d (\sqrt{j} - \sqrt{j-1})^2}}) }{\pi}.
\end{align*}
\end{lemma}
\begin{proof}
    We will show the $\leq$ part by construction. Set $\theta_1 = (1, \sqrt{2} - \sqrt{1}, \dots, \sqrt{d} - \sqrt{d-1}), \theta_2 = (0, 0, \dots, 0)$. 
    According to the Lemma~\ref{lemma:equality_2_for_mr} and the monotonicity of $\arccos$ function, we have
    \begin{align*}
        \inf_{\theta_1, \theta_2} \sup_{\hat{\theta}_1, \hat{\theta}_2} \bigg [ 1 - \frac{\arccos (\frac{(\theta_1-\theta_2) \cdot (\hat{\theta}_1 - \hat{\theta}_2)}{\|\theta_1-\theta_2\|_2 \|\hat{\theta}_1-\hat{\theta}_2\|_2})}{\pi} \bigg ] & \leq \sup_{\hat{\theta}_1, \hat{\theta}_2} \bigg [ 1 - \frac{\arccos (\frac{(\theta_1-\theta_2) \cdot (\hat{\theta}_1 - \hat{\theta}_2)}{\|\theta_1-\theta_2\|_2 \|\hat{\theta}_1-\hat{\theta}_2\|_2})}{\pi} \bigg ] \\
        & = 1 - \frac{\arccos ( \sup_{\hat{\theta}_1, \hat{\theta}_2} \frac{\theta_1 - \theta_2) \cdot (\hat{\theta}_1 - \hat{\theta}_2)}{|\theta_1-\theta_2| |\hat{\theta}_1 - \hat{\theta}_2)|}) }{\pi} \\
        & = 1 - \frac{\arccos(\frac{1}{\sqrt{\sum_{j=1}^d (\sqrt{j} - \sqrt{j-1})^2}}) }{\pi}.
    \end{align*}

\end{proof}

\begin{lemma}\label{lemma:inequality_2_for_mr}
\begin{align*}
    \inf_{\theta_1, \theta_2} \sup_{\hat{\theta}_1, \hat{\theta}_2} \bigg [ 1 - \frac{\arccos (\frac{(\theta_1 - \theta_2) \cdot (\hat{\theta}_1 - \hat{\theta}_2)}{ \|\theta_1-\theta_2\|_2 \|\hat{\theta}_1-\hat{\theta}_2\|_2}) }{ \pi } \bigg ] \geq 1 - \frac{\arccos (\frac{1}{\sqrt{\sum_{j=1}^d (\sqrt{j} - \sqrt{j-1})^2}}) }{\pi}.
\end{align*}
\end{lemma}
\begin{proof}
    Proof by contradiction.
    Assume there exists $\theta_1^*, \theta_2^*$ such that the LHS is smaller than $1 - \frac{\arccos (\frac{1}{\sqrt{\sum_{j=1}^d (\sqrt{j} - \sqrt{j-1})^2}}) }{\pi}$, by monotonicity of cosine function we know 

    \begin{equation*}
        C_0 := \sup_{\htheta_1, \htheta_2} \frac{(\theta_1^* - \theta_2^*) \cdot (\htheta_1 - \htheta_2)}{|\theta_1^* - \theta_2^*| |\htheta_1 - \htheta_2|} < \frac{1}{\sqrt{\sum_{j=1}^d (\sqrt{j} - \sqrt{j-1})^2}} =: C_d.
    \end{equation*}

    Denote $\Delta \theta^* = \theta_1^* - \theta_2^*$. Without loss of generality, we assume $\Delta \theta^* \in [0, 1]^d, \|\Delta \theta^*\|_2 =1 $ and
    \begin{equation*}
        \Delta \theta_1^* \geq \Delta \theta_2^* \geq \dots \geq \Delta \theta_d^*.
    \end{equation*}
    Starting from $\Delta \theta_1^*, \dots, \Delta \theta_d^*$, we construct another feasible solution $\Delta \theta_1, \dots, \Delta \theta_d$ without increasing the corresponding supremum value beyond $C_0$. However, if we compare $\Delta \theta_1, \dots, \Delta \theta_d$ element-wisely with $(\sqrt{k} - \sqrt{k-1})C_0, 1 \leq k \leq d$, the first $\Delta \theta_k$ that is not equal to $(\sqrt{k} - \sqrt{k-1})C_0$ is greater than $(\sqrt{k} - \sqrt{k-1})C_0$, this gives us the contradiction to the $C_0$'s definition. 
    
    Also notice $\sum_{i=1}^d (\Delta \theta_i^*)^2 = 1, C_0 < C_d$, there always exists index $k$ satisfying $\Delta \theta_k^* > (\sqrt{k} - \sqrt{k-1})C_0$. 
    
    Assume the first $\theta_i$ that is not equal to $(\sqrt{i} - \sqrt{i-1})C_0$ is still smaller than $(\sqrt{i} - \sqrt{i-1})C_0$. By the above paragraph, we can find the first $\theta_{k}, k > i$ with $\theta_k^* > (\sqrt{k} - \sqrt{k-1})C_0$. 
    
    Then we adjust $(\theta_i^*, \theta_k^*)$ to $((\sqrt{i} - \sqrt{i-1})C_0, \sqrt{(\theta_i^*)^2 + (\theta_k^*)^2 - (\sqrt{i} - \sqrt{i-1})^2C_0^2})$. We can verify that the assumed inequalities continue to hold. (There are cases for $(\theta_i^*)^2 + (\theta_k^*)^2 - (\sqrt{i} - \sqrt{i-1})^2C_0^2 \leq (\sqrt{k} - \sqrt{k-1})^2C_0^2$, then we just end this modification with $(\sqrt{(\theta_i^*)^2 + (\theta_k^*)^2 - (\sqrt{k} - \sqrt{k-1})^2C_0^2}, (\sqrt{k} - \sqrt{k-1})C_0)$ and then repeat the procedure. )

    As the above procedure repeats, number $\#\{k | \theta_k = (\sqrt{k} - \sqrt{k-1})C_0\}$ is strictly increased. When it stopped, the first non-$(\sqrt{k} - \sqrt{k-1})C_0$ term is larger than the $(\sqrt{k} - \sqrt{k-1})C_0$ and this gives us the contradiction. 
\end{proof}

\end{document}